\documentclass[letterpaper, 10 pt, journal, twoside]{IEEEtran}

\usepackage{amsmath,amsfonts}
\usepackage{array}
\usepackage[caption=false,font=normalsize,labelfont=sf,textfont=sf]{subfig}
\usepackage{textcomp}
\usepackage{stfloats}
\usepackage{url}
\usepackage{verbatim}
\usepackage{graphicx}
\def\BibTeX{{\rm B\kern-.05em{\sc i\kern-.025em b}\kern-.08em
    T\kern-.1667em\lower.7ex\hbox{E}\kern-.125emX}}
\usepackage{balance}

\usepackage{caption}
\usepackage{multirow}

\usepackage{cite}

\usepackage{algorithm}
\usepackage{algpseudocode}
\usepackage{algorithmicx}

\algdef{SE}[DOWHILE]{Do}{doWhile}{\algorithmicdo}[1]{\algorithmicwhile\ #1}%

\usepackage{amsthm}
\makeatletter
\def\@endtheorem{\endtrivlist}
\makeatother
\newtheorem{theorem}{Theorem}
\newtheorem{corollary}{Corollary}
\newtheorem{lemma}{Lemma}
\usepackage{amssymb}

\begin{document}

\markboth{IEEE Robotics and Automation Letters. Preprint Version. Accepted January, 2022}{Kim \MakeLowercase{\textit{et al.}}: TRC: Trust Region Conditional Value at Risk for Safe Reinforcement Learning} 

\title{TRC: Trust Region Conditional Value at Risk 

for Safe Reinforcement Learning}

\author{Dohyeong Kim and Songhwai Oh 
\thanks{Manuscript received: September, 9, 2021; Revised November, 26, 2021; Accepted January, 2, 2022.}
\thanks{This paper was recommended for publication by Editor Jens Kober upon evaluation of the Associate Editor and Reviewers’ comments.
This work was supported in part by the Institute of Information \& Communications Technology Planning \& Evaluation under Grant 2019-0-01190, [SW Star Lab] Robot Learning: Efficient, Safe, and Socially-Acceptable Machine Learning, and in part by the National Research Foundation under Grant NRF-2017R1A2B2006136, both funded by the Korea government (MSIT). \textit{(Corresponding authors: Songhwai Oh.)}}
\thanks{D. Kim and S. Oh are with the Department of Electrical and Computer Engineering and ASRI, Seoul National University, Seoul 08826, Korea (e-mail: dohyeong.kim@rllab.snu.ac.kr, songhwai@snu.ac.kr).}
\thanks{Digital Object Identifier (DOI): 10.1109/LRA.2022.3141829}
\thanks{\copyright 2022 IEEE.  Personal use of this material is permitted.  Permission from IEEE must be obtained for all other uses, in any current or future media, including reprinting/republishing this material for advertising or promotional purposes, creating new collective works, for resale or redistribution to servers or lists, or reuse of any copyrighted component of this work in other works.}
}

\maketitle
\begin{abstract}
As safety is of paramount importance in robotics, reinforcement learning that reflects safety, called \emph{safe RL}, has been studied extensively.
In safe RL, we aim to find a policy which maximizes the desired return while satisfying the defined safety constraints.
There are various types of constraints, among which constraints on conditional value at risk (CVaR) effectively lower the probability of failures caused by high costs since CVaR is a conditional expectation obtained above a certain percentile.
In this paper, we propose a trust region-based safe RL method with CVaR constraints, called \emph{TRC}.
We first derive the upper bound on CVaR and then approximate the upper bound in a differentiable form in a trust region.
Using this approximation, a subproblem to get policy gradients is formulated, and policies are trained by iteratively solving the subproblem.
TRC is evaluated through safe navigation tasks in simulations with various robots and a sim-to-real environment with a Jackal robot from Clearpath.
Compared to other safe RL methods, the performance is improved by 1.93 times while the constraints are satisfied in all experiments.
\end{abstract}


\section{Introduction}
\IEEEPARstart{S}{afety} is one of the top priorities when designing a robot controller.
To this end, several reinforcement learning methods considering safety, called \emph{safe RL}, have been proposed in the field of robotics.
Gangapurwala et al. \cite{gangapurwala2020guided} have proposed a safe RL method for learning quadruped locomotion more stably than traditional RL methods by defining constraints on the robot states, such as foot position.
In Bharadhwaj et al. \cite{bharadhwaj2020conservative}, constraints are defined to locate a given object within a boundary for safe manipulation and robot arm controllers are trained to move the object within the region inside the boundary.
In general, safe RL solves an RL problem while satisfying explicit constraints, which can be formulated as a constrained Markov decision process (CMDP) \cite{altman1999constrained}.

CMDP is a Markov decision process (MDP), where cost functions are additionally defined to provide constraints of the problem.
In a stochastic setting, costs are random variables like rewards and constraints are often provided using the expectation of the cost sum.
For example, to prevent a robot from entering hazard regions, you can define a cost function which determines whether the robot is in the hazard regions and set a constraint that limits the expected cost sum, meaning the expected number of entries into the hazard regions, to be less than a threshold.
However, expectation-based constraints are hard to distinguish risky policies from safe policies.
Suppose that you have two policies and their cost sums follow Gaussian distributions with the same mean but with different variances.
Provided that a given environment fails when the cost sum exceeds a certain level, the policy with the high variance has a higher probability of failures than the one with the low variance. 
Then, expectation-based constraints cannot differentiate these two policies since they have the same expected cost sum.
If constraints are defined on the expectation of the tail, e.g., conditional value at risk (CVaR), rather than the whole distribution, a risky policy can be effectively distinguished.
CVaR is the conditional expectation of a random variable above a certain percentile level and is widely used in financial risk management \cite{rockafellar2000optimization}.
As CVaR differentiates the shape of the tail of the distribution, defining constraints with CVaR avoids risky policies (for more detail, see \cite{chow2014cvar}).

Safety performance is affected by not only how to define the constraint, but also how to update a policy to maximize returns while satisfying constraints.
Yang et al. \cite{yang2021wcsac} proposed a CVaR-constrained RL method called worst-case soft actor-critic (WCSAC).
WCSAC uses a Lagrangian method which relaxes a constrained problem to an unconstrained problem using Lagrange multipliers. 
Lagrangian methods are widely used in safe RL methods \cite{stooke2020responsive, ding2021provably, tessler2018reward, srinivasan2020learning, yang2021wcsac, chow2017risk}, but due to oscillations of Lagrange multipliers, the training process may become unstable \cite{stooke2020responsive}.
Alternatively, a trust-region method is an appropriate tool to solve safe RL problems as in \cite{gangapurwala2020guided, achiam2017constrained, Yang2020Projection-Based, liu2020ipo}.
Achiam et al. \cite{achiam2017constrained} proposed a novel method, called constrained policy optimization (CPO), which solves safe RL problems using linear and quadratic constrained linear programming (LQCLP) within the trust region as an extended study of trust region policy optimization (TRPO) \cite{schulman2015trust}.
CPO has shown excellent performance in safety-sensitive environments while satisfying expectation-based constraints.

We propose a trust region-based method for solving CVaR-constrained RL problems, called \emph{TRC}.
The main problem is to maximize the discounted reward sum while limiting the CVaR of the discounted cost sum not to exceed a given limit value.
To solve the problem using the trust-region method, it is necessary to estimate CVaR of any policy within the trust region in a differentiable form, which is a challenging task.
To resolve this issue, we derive the upper bound on CVaR and replace the CVaR constraint with the upper bound.
To this end, we first assume that the discounted cost sum follows a Gaussian distribution and formulate CVaR in a closed form as in \cite{yang2021wcsac, tang2020worst}. 
Then, we derive the upper bound on the square of the discounted cost sum and extend it to obtain the upper bound on CVaR, which can be approximated in a differentiable form within a trust region.
Using this approximation, a CVaR-constrained subproblem is constructed.
Finally, an optimal policy is obtained by iteratively solving the subproblem with LQCLP as in \cite{achiam2017constrained}.
In addition, to reduce the variance of policy gradient estimate, we use generalized advantage estimations (GAEs) \cite{schulman2015high} instead of advantage functions.

TRC is evaluated on safe navigation tasks with various robots in simulation and a Jackal robot from Clearpath \cite{clearpath2015jackal} in sim-to-real environments.
The experiment results show that TRC improves performance by 1.93 times compared to baseline methods, satisfying CVaR constraints for all tasks.
In conclusion, our main contributions are threefold.
First, we derive the upper bound on CVaR and the approximation of the upper bound within the trust region.
Second, we propose the policy and value update rules for the CVaR-constrained problem.
Finally, the proposed method shows excellent performance in simulation and real-world environments while satisfying constraints.

\section{RELATED WORK}

Garc{\'\i}a et al. \cite{garcia2015comprehensive} have surveyed safe RL methods and classified them into two categories: \emph{optimization criterion} and \emph{exploration process}.
The optimization criterion methods propose policy update rules to satisfy constraints, and the exploration process methods synthesize safe policies by introducing additional devices such as a recovery policy \cite{thananjeyan2021recovery} during exploration.
The optimization criterion methods have the advantage that no additional process is required during exploration, and our method proposes a new CVaR-related policy update rule, so we focus on the optimization criterion.
In this section, we divide the optimization criterion methods into two categories depending on how to update a policy: \emph{Lagrangian} and \emph{trust-region} methods. 

Lagrangian methods can treat safe RL problems as unconstrained problems by introducing Lagrange multipliers.
A simple formulation of this method allows the use of different types of constraints, such as expected cost sums \cite{stooke2020responsive, ding2021provably, tessler2018reward, srinivasan2020learning} or risk measures \cite{yang2021wcsac, chow2017risk, ying2021towards}.
For risk measure constraints, Yang et al. \cite{yang2021wcsac} estimate CVaR through value functions and find policy gradients using soft-actor critic frameworks.
Chow et al. \cite{chow2017risk} and Ying et al. \cite{ying2021towards} also proposed methods for CVaR-constrained RL.
Both methods estimate CVaR using a slack variable formulated by Rockafellar et al. \cite{rockafellar2000optimization} and handle constraints using the Lagrangian method.
However, the Lagrangian method can cause unstable training due to dual gradient descents on the multipliers and policy \cite{stooke2020responsive}, requiring additional techniques such as smoothing\footnote{Smoothing has been implemented using a soft-plus function in the Safety Starter Agents repository of OpenAI \cite{ray2019benchmarking}.} or filtering \cite{stooke2020responsive}.

Trust-region methods obtain policy gradients by linear approximation of the objective within the trust region \cite{schulman2015trust}, and there exist several methods \cite{achiam2017constrained, Yang2020Projection-Based, liu2020ipo} depending on how the constraints are handled.
Constraints are linearly approximated in \cite{achiam2017constrained} or integrated into the objective function using log-barrier functions in \cite{liu2020ipo}. 
Trust-region methods have the advantage of monotonically improving performance, but it is difficult to use with risk measure-related constraints because it requires an upper bound on the constraints in the trust region.
Bisi et al. \cite{bisi2020risk} have proposed not a safe RL but a risk-averse RL method which maximizes mean-variance, one of the risk measures, using a trust-region method.


\section{BACKGROUND}


\subsection{Constrained Markov Decision Process}

A constrained Markov decision process (CMDP) is expressed as a tuple {\small$(\mathcal{S}, \mathcal{A}, \rho, \mathcal{P}, R, C, \gamma)$}, with state space {\small$\mathcal{S} \subset \mathbb{R}^n$}, action space {\small$\mathcal{A} \subset \mathbb{R}^m$}, initial state distribution {\small$\rho$}, transition model {\small$\mathcal{P}:\mathcal{S} \times \mathcal{A} \times \mathcal{S} \mapsto \mathbb{R}$}, reward function {\small$R:\mathcal{S} \times \mathcal{A} \times \mathcal{S} \mapsto \mathbb{R}$}, cost function {\small$C:\mathcal{S} \times \mathcal{A} \times \mathcal{S} \mapsto \mathbb{R}_{\geq0}$}, and a discount factor {\small$\gamma \in [0,1)$}. 
In a CMDP, an agent can interact with the environment through a policy {\small$\pi(\cdot|s)$} that provides a distribution over actions given the state $s$.
Then, the value, action-value, and advantage function are expressed as follows:
\begin{equation}
\small
\label{eq:value functions}
\begin{aligned}
V^{\pi}(s)&:=\mathbb{E}_{\pi, \mathcal{P}}\left[\sum_{t=0}^{\infty}\gamma^t R(s_t, a_t, s_{t+1})|s_0=s\right], \\
Q^{\pi}(s, a)&:=\mathbb{E}_{\pi, \mathcal{P}}\left[\sum_{t=0}^{\infty}\gamma^t R(s_t, a_t, s_{t+1})|s_0=s, a_0=a\right], \\
A^{\pi}(s,a)&:=Q^{\pi}(s,a) - V^{\pi}(s).
\end{aligned}
\end{equation}
As in \cite{achiam2017constrained}, the cost value {\small$V_C^{\pi}$}, cost action-value {\small$Q_C^{\pi}$}, and cost advantage {\small$A_C^{\pi}$} are defined by replacing the reward in (\ref{eq:value functions}) with the cost.
Then, the objective of the agent is to maximize the expected reward sum {\small$J(\pi) := \mathbb{E}_{s \sim \rho}\left[V^{\pi}(s)\right]$} while satisfying constraints consisting of the cost {\small$C$}.
To define constraints, we represent the discounted cost sum as defined in \cite{yang2021wcsac}:
\begin{equation}
\small
\label{eq:discounted cost sum}
\begin{aligned}
C_{\pi} &:= \sum_{t=0}^{\infty}\gamma^t C(s_t, a_t, s_{t+1}), \\
\end{aligned}
\end{equation}
where {\small$s_0 \sim \rho$}, {\small$a_t \sim \pi(\cdot|s_t)$}, and {\small$s_{t+1} \sim \mathcal{P}(\cdot|s_t, a_t)$} for {\small$\forall t$}.
Since {\small$C_{\pi}$} is a random variable in a stochastic setting, safety constraints can be constructed using appropriate probabilities or expectations of the discounted cost sum {\small$C_{\pi}$}.

\subsection{Conditional Value at Risk}

CVaR is one of the representative risk measures used to analyze the tails of distributions in financial portfolios \cite{rockafellar2000optimization}.
Given the cumulative density function (CDF) on a variable $X$, CVaR is obtained by calculating the expectation only for the region where CDF value is above a specific risk level $\alpha$.
\begin{equation}
\small
\label{eq:CVaR}
\begin{aligned}
\mathrm{CVaR}_{\alpha}(X) = \mathbb{E}[X| X \geq \mathrm{ICDF}(1 - \alpha)],
\end{aligned}
\end{equation}
where {\small$\mathrm{ICDF}$} is the inverse cumulative density function.
If the variable {\small$X$} follows a Gaussian distribution {\small$\mathcal{N}(\mu, \sigma)$}, CVaR can be expressed in a simple closed-form as follows:
\begin{equation}
\small
\label{eq:CVaR on Gaussian}
\begin{aligned}
\mathrm{CVaR}_{\alpha}(X) = \mu + \frac{\phi(\Phi^{-1}(\alpha))}{\alpha}\sigma,
\end{aligned}
\end{equation}
where {\small$\phi(x) = \frac{1}{\sqrt{2\pi}}e^{-\frac{x^2}{2}}$} and {\small$\Phi(x) = \frac{1}{2}\left(1 + \mathrm{erf}(\frac{x}{\sqrt{2}})\right)$} \cite{khokhlov2016conditional}.
For general distribution, CVaR can be estimated from sampling, which is computationally expensive.
Therefore, to provide a practical method, we assume that {\small$C_{\pi}$} follows a Gaussian distribution to utilize the closed-form (\ref{eq:CVaR on Gaussian}) as commonly used in \cite{tang2020worst, yang2021wcsac}.
To get the mean and variance of the distribution over {\small$C_{\pi}$}, the cost square function {\small$S_{C}^{\pi}$} is defined as follows:
\begin{equation}
\label{eq:cost square and variance function}
\small
\begin{aligned}
S_{C}^{\pi}(s) &:= \mathbb{E}_{\pi, \mathcal{P}}\left[C_{\pi}^2|s_0 = s\right], \\ 
S_{C}^{\pi}(s, a) &:= \mathbb{E}_{\pi, \mathcal{P}}\left[C_{\pi}^2|s_0 = s, a_0 = a\right]. \\ 
\end{aligned}
\end{equation}
Additionally, the cost square advantage function is defined as {\small${A_S^{\pi}(s,a):=S_C^{\pi}(s,a) - S_C^{\pi}(s)}$}.
The expectation of the discounted cost sum {\small$C_{\pi}$} and square of the discounted cost sum {\small$C_{\pi}^2$} are denoted as {\small$J_C(\pi):=\mathbb{E}_{s \sim \rho}\left[V_C^{\pi}(s)\right]$} and {\small$J_{S}(\pi):=\mathbb{E}_{s \sim \rho}\left[S_C^{\pi}(s)\right]$}, respectively.
Then, the discounted cost sum can be expressed as {\small$C_{\pi} \sim \mathcal{N}(J_{C}(\pi), J_{S}(\pi) - J_{C}(\pi)^2)$} \cite{tang2020worst}.
Finally, the CVaR of {\small$C_{\pi}$} can be approximated as follows \cite{tang2020worst}:
\begin{equation}
\label{eq:CVaR of discounted cost sum}
\small
\begin{aligned}
\mathrm{CVaR}_{\alpha}(C_{\pi}) \approx J_{C}(\pi) + \frac{\phi(\Phi^{-1}(\alpha))}{\alpha}\sqrt{J_{S}(\pi) - J_{C}(\pi)^2}.
\end{aligned}
\end{equation}

\subsection{Constrained Policy Optimization}

Constrained Policy Optimization (CPO) \cite{achiam2017constrained} is a trust region-based method to solve an expectation-constrained RL problem and the problem is written as follows:
\begin{equation}
\small
\label{eq:CPO}
\begin{aligned}
&\; \underset{\pi}{\mathrm{maximize}} \underset{\rho, \pi, \mathcal{P}}{\mathbb{E}}\left[\sum_{t=0}^{\infty}\gamma^t R(s_t,a_t, s_{t+1})\right] \\
&\mathbf{s.t.} \quad \underset{\rho, \pi, \mathcal{P}}{\mathbb{E}}\left[\sum_{t=0}^{\infty}\gamma^t C(s_t, a_t, s_{t+1})\right] \leq \frac{d}{1-\gamma}, \\
\end{aligned}
\end{equation}
where $d$ is a limit value for the safety constraint.
Achiam et al. \cite{achiam2017constrained} derives the following subproblem to update policy {\small$\pi'$} within the trust region of policy {\small$\pi$}.
{\small
\begin{align}
\label{eq:CPO approximation}
&\qquad \underset{\pi'}{\mathrm{maximize}}\underset{\begin{subarray}{c}s\sim d^{\pi}\\a\sim\pi\end{subarray}}{\mathbb{E}}\left[\frac{\pi'(a|s)}{\pi(a|s)}A^{\pi}(s,a)\right] \\
\mathbf{s.t.} \quad &\underset{s \sim \rho}{\mathbb{E}}\left[V_C^{\pi}(s)\right] + \frac{1}{1-\gamma}\underset{\begin{subarray}{c}s\sim d^{\pi}\\a\sim\pi\end{subarray}}{\mathbb{E}}\left[\frac{\pi'(a|s)}{\pi(a|s)}A_{C}^{\pi}(s,a)\right] \leq \frac{d}{1-\gamma}, \nonumber\\
& \qquad \;\; \underset{s\sim d^{\pi}}{\mathbb{E}}\left[D_{\mathrm{KL}}(\pi||\pi')[s]\right] \leq \delta, \nonumber
\end{align}
} 
\!\!where {\small$D_\mathrm{KL}$} is the Kullback-Leibler divergence, {\small$d^{\pi}(s):=(1-\gamma)\sum_{t=0}^{\infty}\gamma^{t}\mathrm{Prob}(s_t=s|\pi)$} is the discounted state distribution.
Then, a suboptimal policy is obtained by iteratively solving the subproblem (\ref{eq:CPO approximation}) with linear approximations on the objective and the safety constraint and quadratic approximations on the KL divergence term.

\section{PROPOSED METHOD}

The proposed method utilizes the trust-region method and addresses a safe RL problem with CVaR constraints, which are more conservative than the expectation of discounted cost sums in that CVaR focuses on the tail of the distribution.
The CVaR-constrained problem is formulated as:
\begin{equation}
\label{eq:main problem}
\small
\begin{aligned}
\underset{\pi}{\mathrm{maximize}}& \underset{\rho, \pi, \mathcal{P}}{\mathbb{E}}\left[\sum_{t=0}^{\infty}\gamma^t R(s_t, a_t, s_{t+1})\right] \\
\mathbf{s.t.} \quad &\mathrm{CVaR}_{\alpha}(C_{\pi}) \leq d/(1 - \gamma). \\
\end{aligned}
\end{equation}
In this section, the upper bound on the CVaR within the trust region is derived first. 
Next, with the upper bound, a trust region-based subproblem for policy update is proposed and the proposed policy update rule is described.
Then, GAE \cite{schulman2015high} for the cost square value, newly defined in this paper, is introduced to use instead of advantage, and finally, the value update rules are described.

\subsection{Upper Bound on CVaR}

This section presents the upper bound on the CVaR of a new policy {\small$\pi'$} from a given policy {\small$\pi$}.
To this end, we first define useful functions and establish a theorem for the square of the discounted cost sum.
Then, the upper bound on the CVaR is obtained by combining this derived theorem with result on the discounted cost sum derived in \cite{achiam2017constrained}.

Similar to the discounted state distribution {\small$d^{\pi}$}, the doubly discounted state distribution is defined as {\small$d_2^{\pi}(s):=(1-\gamma^2)\sum_{t=0}^{\infty}\gamma^{2t}\mathrm{Prob}(s_t=s|\pi)$}.
Then, the expectation of the discounted cost sum and square of the discounted cost sum can be rewritten as follows\footnote{For brevity, {\small${\underset{\begin{subarray}{c}s \sim d^{\pi}\\a \sim \pi'\\s' \sim \mathcal{P}\end{subarray}}{\mathbb{E}}}$} is denoted by {\small$\underset{d^{\pi},\pi',\mathcal{P}}{\mathbb{E}}$} from now on.}:
{\small
\begin{align}
\label{eq:expectation of the discounted cost sum}
J_{C}(\pi) &= \underset{\tau \sim \pi}{\mathbb{E}}\left[\sum_{t=0}^{\infty}\gamma^t C_t\right] = \frac{1}{1-\gamma}\underset{d^{\pi}, \pi, \mathcal{P}}{\mathbb{E}}\left[C(s, a, s')\right], \nonumber\\
J_{S}(\pi) &= \underset{\tau \sim \pi}{\mathbb{E}}\left[\left(\sum_{t=0}^{\infty}\gamma^t C_t\right)^2\right] \\
&= \underset{\tau \sim \pi}{\mathbb{E}}\left[\sum_{t=0}^{\infty}\gamma^{2t}C_t^2 + 2\gamma\sum_{t=0}^{\infty}\left\{\gamma^{2t}C_t\sum_{k=t+1}^{\infty}\gamma^{k-t-1} C_k\right\}\right] \nonumber\\
&= \frac{1}{1-\gamma^2}\underset{d_2^{\pi}, \pi, \mathcal{P}}{\mathbb{E}}\left[C(s, a, s')^2 + 2\gamma C(s, a, s')V_C^{\pi}(s')\right], \nonumber
\end{align}
}
\!\!where {\small$C_t=C(s_t,a_t,s_{t+1})$}.
Using (\ref{eq:expectation of the discounted cost sum}), the upper bound on the difference in {\small$J_S$} between {\small$\pi$} and {\small$\pi'$} is derived as follows.
\begin{theorem}
\label{theorem:cost square sum}
For any policy {\small$\pi$} and {\small$\pi'$}, define a variable:
\begin{equation*}
\resizebox{\columnwidth}{!}{%
$\epsilon_S^{\pi'}:=\frac{\gamma^2}{1-\gamma^2}\underset{s}{\mathrm{max}}\underset{a \sim \pi'}{\mathbb{E}}\left[A_S^{\pi}(s,a)\right] + \frac{2\gamma\underset{s}{\mathrm{max}}\left|V_C^{\pi}(s)\right|}{1-\gamma}\underset{d_2^{\pi}, \pi', \mathcal{P}}{\mathbb{E}}\left[C(s,a,s')\right].$%
}
\end{equation*}
Then, the following inequality holds:
\begin{equation}
\small
\label{eq:cost square upper bound}
\begin{aligned}
J_{S}(\pi') - J_{S}(\pi) \leq& \frac{1}{1-\gamma^2}\underset{\begin{subarray}{c}s \sim d_2^{\pi}\\a \sim \pi\end{subarray}}{\mathbb{E}}\left[\frac{\pi'(a|s)}{\pi(a|s)}A_{S}^{\pi}(s,a) \right] + \\
&\frac{2\epsilon_S^{\pi'}}{1-\gamma^2}\underset{s}{\mathrm{max}}D_{\mathrm{TV}}(\pi'||\pi)[s], \\
\end{aligned}
\end{equation}
where equality holds if {\small$\pi=\pi'$} and {\small$D_{\mathrm{TV}}$} is the total variation divergence \cite{schulman2015trust}.
\end{theorem}
The proof is given in Appendix B. 
Assuming that the {\small$D_{\mathrm{TV}}$} term is small enough in (\ref{eq:cost square upper bound}), Theorem \ref{theorem:cost square sum} gives a differentiable approximation of the upper bound on $J_S$ by removing the {\small$D_{\mathrm{TV}}$} term.
This assumption is valid if {\small$\pi'$} is within the trust region of {\small$\pi$}.
Before inducing the upper bound on the CVaR, for brevity the following functions are defined:
\begin{equation}
\small
\label{eq:surrogate function definition}
\begin{aligned}
J_C^{\pi}(\pi')&:=J_{C}(\pi) + \frac{1}{1-\gamma}\underset{\begin{subarray}{c}s \sim d^{\pi}\\a \sim \pi\end{subarray}}{\mathbb{E}}\left[\frac{\pi'(a|s)}{\pi(a|s)}A_{C}^{\pi}(s,a) \right], \\
J_S^{\pi}(\pi')&:=J_{S}(\pi) + \frac{1}{1 - \gamma^2}\underset{\begin{subarray}{c}s \sim d_2^{\pi}\\a \sim \pi\end{subarray}}{\mathbb{E}}\left[\frac{\pi'(a|s)}{\pi(a|s)}A_{S}^{\pi}(s, a)\right], \\ D^{\pi}(\pi')&:=\underset{s}{\mathrm{max}}D_{\mathrm{TV}}(\pi'||\pi)[s],
\end{aligned}
\end{equation}
where {\small$J_C^{\pi}(\pi') = J_C(\pi)$} and {\small$J_S^{\pi}(\pi') = J_S(\pi)$} hold when {\small$\pi = \pi'$} as the expectations on the advantages are zero.
Since CVaR is composed of {\small$J_C$} and {\small$J_S$} as in (\ref{eq:CVaR of discounted cost sum}), the upper bound on CVaR can be obtained using the bounds on {\small$J_C$} derived in \cite{achiam2017constrained} and Theorem \ref{theorem:cost square sum} as follows.
\begin{theorem}
\label{theorem:upper bound}
For any policies {\small$\pi$} and {\small$\pi'$}, let define {\small$\epsilon_{C}^{\pi'}:=\underset{s}{\mathrm{max}}\underset{a \sim \pi'}{\mathbb{E}}\left[A_C^{\pi}(s,a)\right]$ and \\
$\epsilon_{\mathrm{CVaR}}^{\pi'}:=\epsilon_S^{\pi'} + \left(J_C^{\pi}(\pi') - \frac{\gamma\epsilon_C^{\pi'}}{(1-\gamma)^2}D^{\pi}(\pi')\right)\frac{2\gamma(1+\gamma)}{1-\gamma}\epsilon_C^{\pi'}$}. \\
Then, the following inequality holds:
{\small
\begin{align}
\label{eq:CVaR upper bound}
&\mathrm{CVaR}_{\alpha}(C_{\pi'}) \leq J_C^{\pi}(\pi') + \frac{\phi(\Phi^{-1}(\alpha))}{\alpha}\sqrt{J_S^{\pi}(\pi') - J_C^{\pi}(\pi')^2} \;+ \nonumber\\
&\; \frac{2}{1-\gamma}\left(\frac{\gamma\epsilon_C^{\pi'}}{1-\gamma} + \frac{\phi(\Phi^{-1}(\alpha))/\alpha}{{\sqrt{J_S(\pi') - J_C(\pi')^2}}}\frac{\epsilon_{\mathrm{CVaR}}^{\pi'}}{1+\gamma}\right) D^{\pi}(\pi'),
\end{align}
}
\!\!where equality holds if {\small$\pi=\pi'$}.
\end{theorem}
The proof is given in Appendix C.
Assuming that {\small$D^{\pi}(\pi')$} is small enough, Theorem \ref{theorem:upper bound} yields a differentiable approximation of the upper bound on CVaR by removing {\small$D^{\pi}(\pi')$}.

\subsection{Policy Optimization in Trust Region}

This section shows a trust region-based subproblem for the policy update.
We parameterize the policy with {\small$\theta$} and denote {\small$\pi_{\theta_{\mathrm{old}}}$} as {\small$\pi_{\mathrm{old}}$} and {\small$\pi_{\theta}$} as {\small$\pi$} for brevity.
In a trust region, the inequality in (\ref{eq:main problem}) can be replaced with the upper bound on CVaR in (\ref{eq:CVaR upper bound}) as follows:
{\small
\begin{align}
\label{eq:CVaR constraints}
& \frac{d}{1-\gamma} \geq \; J_C^{\pi_{\mathrm{old}}}(\pi) + \frac{\phi(\Phi^{-1}(\alpha))}{\alpha}\sqrt{J_S^{\pi_{\mathrm{old}}}(\pi) - J_C^{\pi_{\mathrm{old}}}(\pi)^2} \;+ \nonumber\\
&\;\; \frac{2}{1-\gamma}\left(\frac{\gamma\epsilon_C^{\pi}}{1-\gamma} + \frac{\phi(\Phi^{-1}(\alpha))/\alpha}{{\sqrt{J_S(\pi) - J_C(\pi)^2}}}\frac{\epsilon_{\mathrm{CVaR}}^{\pi}}{1+\gamma}\right) D^{\pi_{\mathrm{old}}}(\pi).
\end{align}
}\!\!
However, as mentioned in \cite{achiam2017constrained} and \cite{schulman2015trust}, divergence terms in the constraint cause a small step size of the policy update.
Thus, we approximate the upper bound on CVaR by removing the divergence term in (\ref{eq:CVaR constraints}) and add a trust-region constraint as in \cite{achiam2017constrained} and \cite{schulman2015trust}.
Then, the proposed CVaR-constrained subproblem can be written as below by replacing the constraint in (\ref{eq:CPO approximation}) with the approximated CVaR.
{\small
\begin{align}
\label{eq: proposed constrained subproblem}
&\quad\quad\quad\quad \underset{\pi}{\mathrm{maximize}}\;\underset{\begin{subarray}{c}s\sim d^{\pi_{\mathrm{old}}}\\a\sim\pi_{\mathrm{old}}\end{subarray}}{\mathbb{E}}\left[\frac{\pi(a|s)}{\pi_{\mathrm{old}}(a|s)}\mathrm{}A^{\pi_{\mathrm{old}}}(s, a)\right] \\
&\mathrm{s.t.} \quad J_C^{\pi_{\mathrm{old}}}(\pi) + \frac{\phi(\Phi^{-1}(\alpha))}{\alpha}\sqrt{J_S^{\pi_{\mathrm{old}}}(\pi) - J_C^{\pi_{\mathrm{old}}}(\pi)^2} \leq \frac{d}{1 - \gamma}, \nonumber\\
& \quad\quad\quad\quad\quad\quad \underset{s\sim d_{\pi_{\mathrm{old}}}}{\mathbb{E}}\left[D_{\mathrm{KL}}(\pi_{\mathrm{old}}||\pi)[s]\right] \leq \delta. \nonumber
\end{align}
}\!\!
Directly solving (\ref{eq: proposed constrained subproblem}) is difficult as it is non-convex, so we linearly approximate the objective and the CVaR constraint and quadratically approximate the KL divergence as in \cite{achiam2017constrained} to find a practical solution. Then, LQCLP is used to find an optimal policy of the approximated subproblem.
In addition, if the feasible set of the approximated subproblem is empty, the policy is updated to minimize only the approximation of the upper bound on CVaR within the trust region to obtain a safe policy as in CPO \cite{achiam2017constrained}.
Finally, the original constrained problem (\ref{eq:main problem}) can be solved by iteratively solving the subproblem (\ref{eq: proposed constrained subproblem}).

\subsection{Value and Square Function Update}

The tradeoff between bias and variance of policy gradient estimates can be effectively controlled using generalized advantage estimations (GAEs) instead of advantages.
Thus, we also formulate GAE for the newly defined square function, then present the value and square function update rules.

To formulate GAE for the square function, a TD error is first defined as: {\small$\delta_t^{S}:=C_{t+k}^2 + 2\gamma C_{t+k}V_{C,t+k+1}^{\pi} + \gamma^2 S_{C,t+k+1}^{\pi} - S_{C,t+k}^{\pi}$}, where {\small$S_{C,t}^{\pi}=S_C^{\pi}(s_t)$}.
Then, the GAE is derived as follows:
\begin{equation}
\small
\label{eq: GAE for cost square}
\hat{A}_{S,t}^{\mathrm{GAE}(\gamma, \lambda)} := \sum_{i=t}^{\infty}(\gamma^2\lambda)^{i-t}\delta_i^S,
\end{equation}
where {\small$\lambda \in [0, 1]$} is an exponential weight (see Appendix D for its derivation).
The value function is parameterized by a neural network {\small$\phi$} and the cost value function and square function are parameterized by {\small$\phi_C$} and {\small$\psi_C$}, respectively.
For the targets of the value and square functions, we use {\small$\mathrm{TD}(\lambda)$} which can be obtained by adding the current value to {\small$\hat{A}_{t}^{\mathrm{GAE}(\gamma, \lambda)}$}.
Then, the loss functions can be written as in \cite{yang2021wcsac}:
{\small
\begin{align}
\label{eq:target value}
&\qquad\qquad\;\; \underset{\phi}{\mathrm{min}}\underset{\rho, \pi, \mathcal{P}}{\mathbb{E}}\left[\left(V^{\pi, \mathrm{target}}(s_t) - V^{\pi}(s_t;\phi)\right)^2 \right], \nonumber\\
&\qquad\qquad \underset{\phi_C}{\mathrm{min}}\underset{\rho, \pi, \mathcal{P}}{\mathbb{E}}\left[\left(V_C^{\pi, \mathrm{target}}(s_t) - V_C^{\pi}(s_t;\phi_C)\right)^2 \right], \\
&\resizebox{\columnwidth}{!}{%
$\underset{\psi_C}{\mathrm{min}}\underset{\rho, \pi, \mathcal{P}}{\mathbb{E}}\left[S_{C}^{\pi}(s_t;\psi_C) + S_C^{\pi, \mathrm{target}}(s_t) - 2\sqrt{S_{C}^{\pi}(s_t;\psi_C)S_C^{\pi, \mathrm{target}}(s_t)}\right]$%
}, \nonumber
\end{align}}\!\!\!
where {\small$V^{\pi, \mathrm{target}}, V_C^{\pi, \mathrm{target}}$}, and {\small$S_C^{\pi, \mathrm{target}}$} are targets for the value, cost value, and cost square function, respectively.

We have proposed the upper bound on CVaR, the trust-region method for policy update, and the update rules for the value and square networks.
The overall algorithm is summarized in Algorithm \ref{algo: proposed algorithm}.

\begin{algorithm}[!t]
\small
    \caption{TRC}
    \label{algo: proposed algorithm}
    \begin{algorithmic}[1]
    \Require Initial policy network $\pi(a|s;\theta)$, value network $V^{\pi}(s;\phi)$, cost value network $V_C^{\pi}(s;\phi_C)$, cost square network $S_C^{\pi}(s;\psi_C)$.
    \For{epochs=1, P}
        \State Initialize trajectory memory $D$.
        \For{episodes=1, E}
            \State Get an initial state $s_0$ from the environment.
            \For{t=0, T}
                \State Sample an action $a_t \sim \pi(\cdot|s_t;\theta)$.
                \State Feed the action $a_t$ to the environment.
                \State Get reward $r_t$, cost $c_t$, and next state. $s_{t+1}$, and store $(s_t, a_t, r_t, c_t, s_{t+1})$ in $D$.
            \EndFor
        \EndFor
        \State Calculate $\hat{A}^{\mathrm{GAE}(\gamma, \lambda)}$, $\hat{A}_{C}^{\mathrm{GAE}(\gamma, \lambda)}$, and $\hat{A}_{S}^{\mathrm{GAE}(\gamma, \lambda)}$ with $V^{\pi}(s;\phi)$, $V_C^{\pi}(s;\phi_C)$, $S_C^{\pi}(s;\psi_C)$, and $D$.
        \State Calculate a policy gradient from (\ref{eq: proposed constrained subproblem}) using the calculated GAEs as the advantages and update $\pi(a|s;\theta)$.
        \State Update $V^{\pi}(s;\phi)$, $V_C^{\pi}(s;\phi_C)$, and $S_C^{\pi}(s;\psi_C)$ using (\ref{eq:target value}) from $D$.
      \EndFor
    \end{algorithmic}
\end{algorithm}

\section{EXPERIMENTS}

\subsection{Simulation Setup}

The safety gym \cite{ray2019benchmarking} provides various robots and tasks, so it is an environment suitable for measuring performance of safe RL methods.
In our experiments, \emph{goal task}, which is to navigate to a given goal without passing through hazard areas, is performed with three robots, \emph{point}, \emph{car}, and \emph{doggo}, as shown in Figure \ref{fig:safety gym task}.
Eight hazard areas and one goal are randomly spawned at the beginning of each episode, and the goal is respawned if a robot reaches the goal.
The doggo goal task is difficult to train safe RL methods directly, so we first train a low-level controller to reach a given goal in obstacle-free environments.
Then, safe RL agents are trained in the reconstructed environment whose action is to give a two-dimensional subgoal position to the low-level controller.

The state space is a 24-dimensional space which consists of a two-dimensional goal direction, distance to the goal, two-dimensional acceleration, two-dimensional velocity, rotation velocity, and 16-dimensional LiDAR sensors and the action space is a two-dimensional space.
The reward and the cost function are the same for all tasks as follows:
\begin{equation}
\small
\label{eq:reward function}
\begin{aligned}
&R(s, a, s') = d_g(s) - d_g(s') + 
        \begin{cases}
          1, & \text{if}\ d_g(s) \leq \delta \\
          0, & \text{otherwise}
        \end{cases}, \\
&C(s, a, s') = \mathrm{Sigmoid}(w_c\cdot(r_h - d_h(s))), \\
\end{aligned}
\end{equation}
where $d_g$ gives the distance between the given state and the goal, $d_h$ gives the minimum distance between the given state and hazards, $\delta$ is a goal threshold, $w_c$ is a cost weight, and $r_h$ is the size of the hazard.

The policy network has two hidden layers of sizes $(512, 512)$ with \emph{relu} activation and an output layer with \emph{sigmoid} activation. The value and square networks also have the same hidden layers as the policy and the output activation is linear for the value network and \emph{softplus} for the square network.
For constraints, the risk level $\alpha$ is $0.125$ and the limit value $d$ is $0.025$.
For value and square networks, the learning rate is $0.0002$ and $\lambda$ for GAE is $0.97$. 

\begin{figure}[!t]
\centering
\subfloat[Point Goal]%
    {\label{sfig:static point goal}{\includegraphics[width=0.3\linewidth]{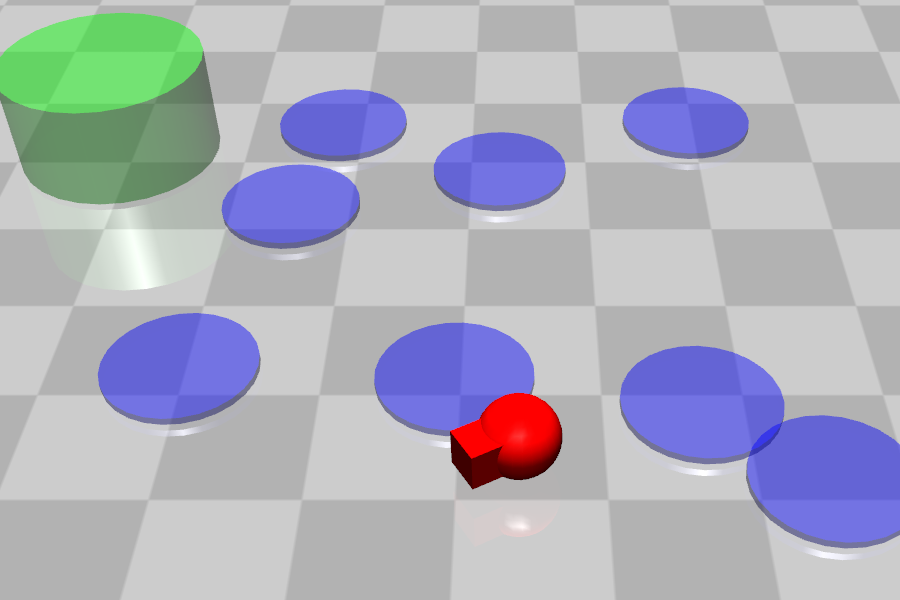}}}\hfill
\subfloat[Car Goal]%
    {\label{sfig:static car goal}{\includegraphics[width=0.3\linewidth]{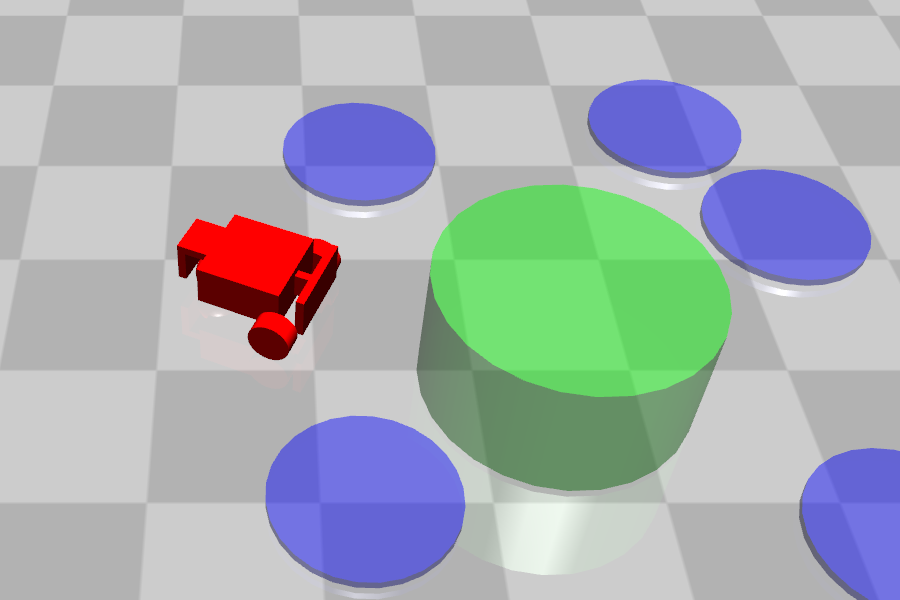}}}\hfill
\subfloat[Doggo Goal]%
    {\label{sfig:doggo goal}{\includegraphics[width=0.3\linewidth]{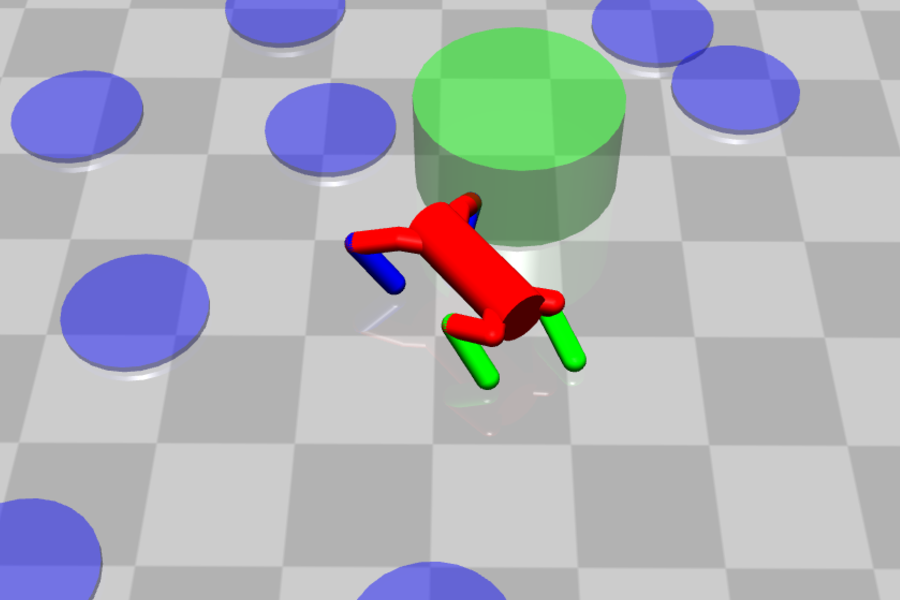}}}\hfill

\subfloat[Jackal Goal]%
    {\label{sfig:jackal goal}{\includegraphics[width=0.3\linewidth]{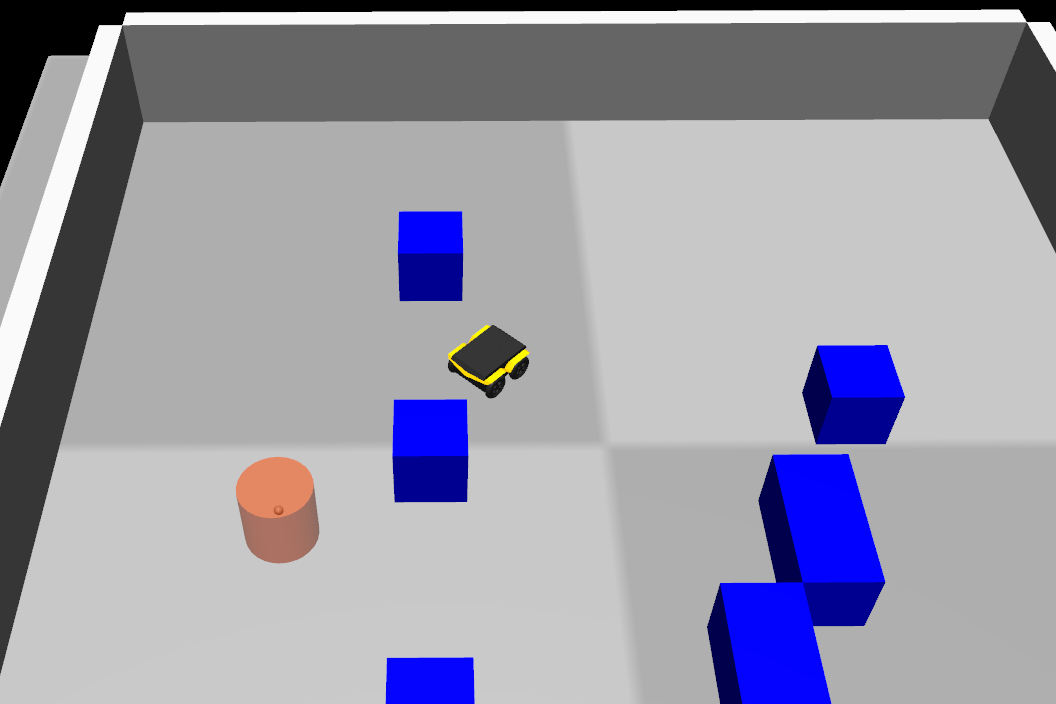}}}\hfill
\subfloat[Jackal robot]%
    {\label{sfig:jackal robot}{\includegraphics[width=0.3\linewidth]{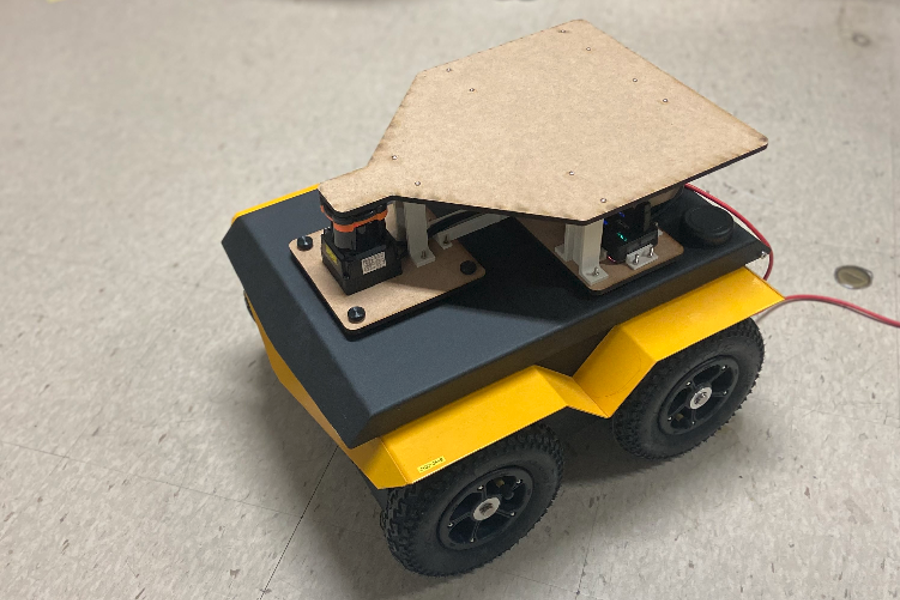}}}\hfill
\subfloat[Real world]%
    {\label{sfig:real environment}{\includegraphics[width=0.3\linewidth]{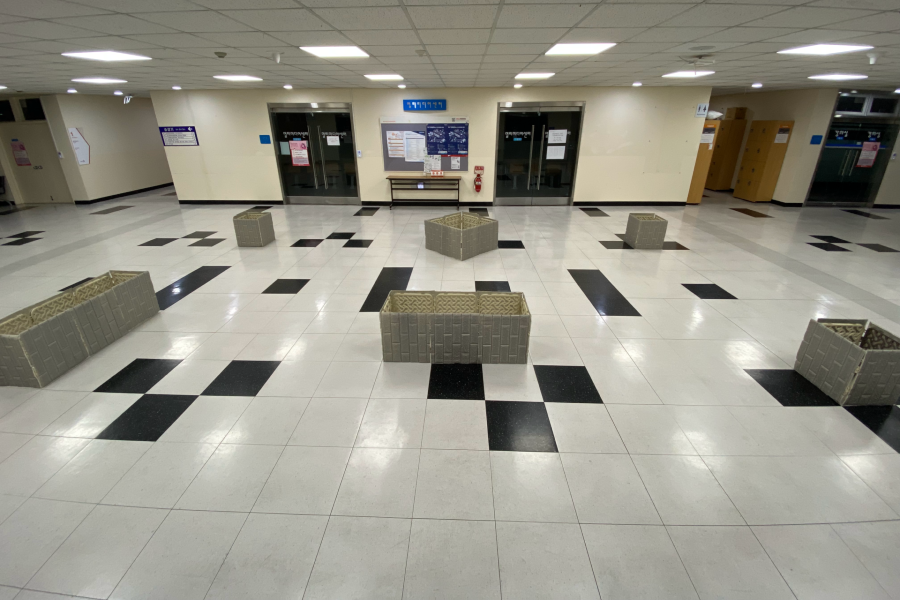}}}\hfill

\caption{Goal tasks of the safety gym and the sim-to-real Jackal environments. In the safety gym (a), (b), and (c), robots, hazards, and goals are indicated in red, purple, and green, respectively. In the Jackal simulation (d), obstacles are indicated in blue, and a goal is indicated in red. The real Jackal robot and the real environment with obstacles are shown in (e), (f).}
\label{fig:safety gym task}
\vspace{-10pt}
\end{figure}

\subsection{Sim to Real Experiment Setup}

\begin{figure*}[t]
\centering
\subfloat[Point goal]%
    {\label{sfig:static point result}{\includegraphics[width=0.2\linewidth]{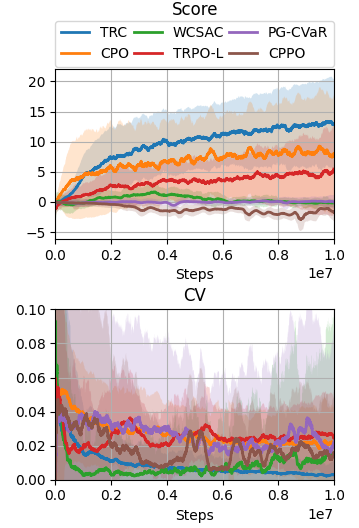}}}\hfill
\subfloat[Car goal]%
    {\label{sfig:static car result}{\includegraphics[width=0.2\linewidth]{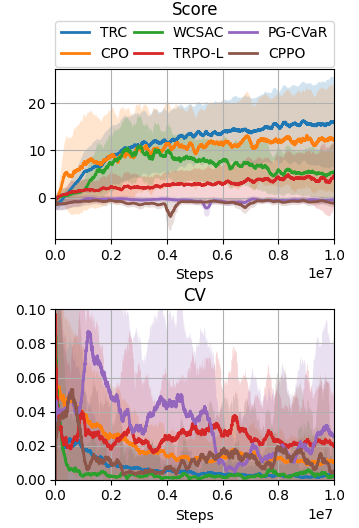}}}\hfill
\subfloat[Doggo goal]%
    {\label{sfig:doggo result}{\includegraphics[width=0.2\linewidth]{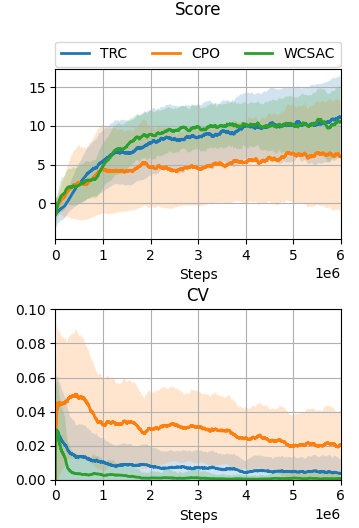}}}\hfill
\subfloat[Jackal goal]%
    {\label{sfig:Jackal result}{\includegraphics[width=0.2\linewidth]{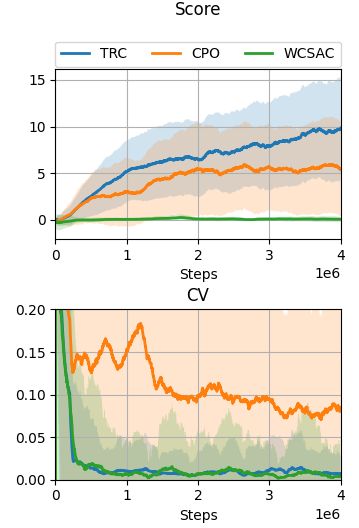}}}\hfill
\subfloat[GAE]%
    {\label{sfig:ablation result}{\includegraphics[width=0.2\linewidth]{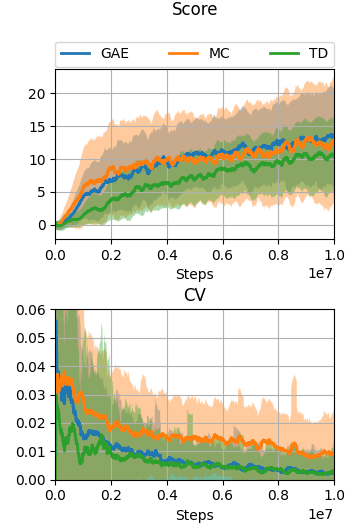}}}\hfill
\vspace{-5pt}
\caption{Training curves of the simulation and sim to real experiments. 
The graph on the top row shows the score for each task and the graph on the bottom shows the average constraint violations (CV) divided by the episode length during training.
Each method is trained with five different random seeds.
The bold line shows the mean value, and the area in lighter color shows the standard deviations.}
\label{fig:total results}
\vspace{-10pt}
\end{figure*}

The sim-to-real experiment is to navigate to a given goal with a Jackal robot \cite{clearpath2015jackal} while avoiding obstacles and walls.
For safety, the Jackal robot is constrained not to have a negative linear velocity since the LiDAR range of the Jackal robot is only 270 degrees.
Additionally, when an agent contacts an obstacle, the episode ends and the agent gets an additional cost equal to the remaining number of episode steps, which makes this task more difficult than the safety gym tasks.
An agent is trained on the Mujoco simulator \cite{todorov2012mujoco} (an example is shown in Figure \ref{sfig:jackal goal}) and controlled using ROS packages in the real environment (shown in Figure \ref{sfig:jackal robot}, \ref{sfig:real environment}). 
Evaluation of policy performance in the real environment is conducted without further training.
To estimate a relative goal direction and distance, we use a SLAM method, Cartographer \cite{cartographer}, with a map without obstacles.

The state space is a 31-dimensional space which consists of a two-dimensional goal direction, goal distance, linear and angular velocity, and 26-dimensional LiDAR sensors. 
The Jackal robot moves with two-dimensional commands for linear acceleration and angular velocity and the reward and cost functions are the same as in the simulation experiment.
Also, the network structures and hyperparameters are the same as the simulation setup.

\subsection{Baselines}

For expectation-constrained safe RL methods, CPO \cite{achiam2017constrained} and trust region policy optimization with a Lagrangian method (TRPO-L) \cite{ray2019benchmarking} are used.
For CVaR-constrained safe RL methods, worst-case soft actor critic (WCSAC) \cite{tang2020worst}, policy gradient with CVaR (PG-CVaR) \cite{chow2017risk}, and CVaR proximal policy optimization (CPPO) \cite{ying2021towards} are used.
Comparing methods with different types of safety constraints is tricky.
Therefore, we experiment with the expectation-constrained methods for three different limit values of $0.005$, $0.01$ and $0.025$ and report the best value to compare results. WCSAC, PG-CVaR, and CPPO use the same limit values as TRC.

\subsection{Results}

\begin{figure}[t]
\centering
\includegraphics[width=.75\linewidth]{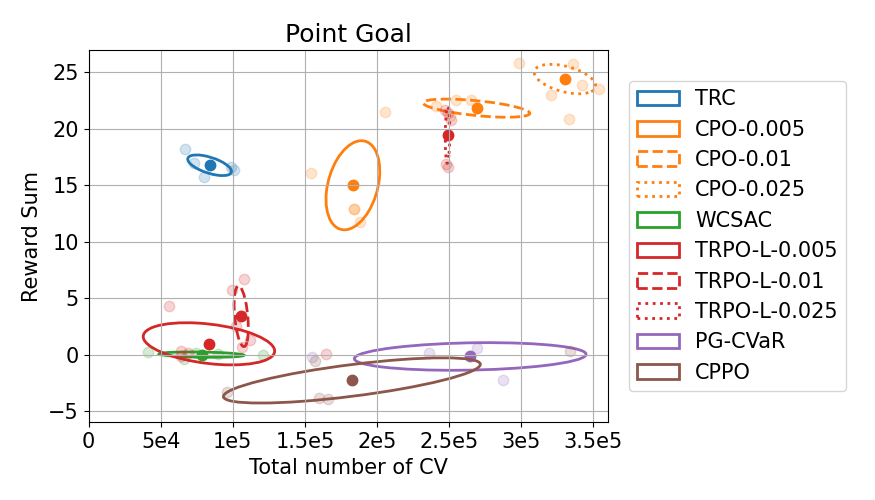}
\vspace{-5pt}
\caption{Total number of CVs and reward sum on the point goal task.
Ellipses are drawn with the center, semi-major, and semi-minor axes obtained from the mean and covariance of the five runs for each method.
Results for CPO and TRPO-L at different thresholds are included, with the threshold value indicated in the legend.
}
\label{fig:pareto results}
\vspace{-10pt}
\end{figure}

The training curves of the simulation experiments are shown in Figure \ref{fig:total results}, and the result video can be found in the attached material.
Each method is trained on five different random seeds for each task.
To show safety performance, the number of entering into the hazard regions, denoted as a constraint violation (CV), is counted per episode, and CVaR is estimated by (\ref{eq:CVaR on Gaussian}) using the mean and variance of the CV.
Figure \ref{fig:pareto results} shows a graph of the total number of CVs during training and the final reward sum.
As the CV decreases, the reward sum tends to decrease, so instead of the reward sum, a score metric is defined as {\small$\sum_{t=0}^{T-1}R_t/\left(1 + \sum_{t=0}^{T-1}\mathrm{CV}_t\right)$}.
The average final score, CV, and CVaR of each task are presented in Table \ref{table:final policy}, and evaluation results of the real Jackal environment are shown in Table \ref{table:real environment}.
Additionally, to analyze the effect of the GAE, we have trained TRC with GAE ($\lambda = 0.97$), Monte-Carlo (MC, $\lambda = 1.0$), and temporal difference (TD, $\lambda = 0.0$) on the point goal task, the results of which are shown in Figure \ref{sfig:ablation result}.

TRC shows the best performance in all tasks as shown in Table \ref{table:final policy}, \ref{table:real environment}.
The performances are improved by 1.55 times, 1.21 times, 1.06 times, and 3.91 times compared to the second best performing algorithm in the point, car, doggo, and Jackal tasks, respectively, which is 1.93 times the performance of the second best method on average.
The constraints are also satisfied in all tasks seeing that CVaR values in Table \ref{table:final policy} and \ref{table:real environment} are below the limit value of $0.025$.
In Figure \ref{fig:pareto results}, TRC is located in the most upper left corner, which means that it shows high reward sums while having low CVs, and has a small covariance, which can be interpreted as TRC can train policies robustly against random seeds.
WCSAC shows a score as high as TRC while satisfying the constraint in the doggo task but has low scores in the other tasks.
It seems that WCSAC can have a high score in the doggo task because the task difficulty is lowered by using a low-level controller.
However, if there are significant constraint violations in the early training, Lagrange multipliers rapidly increase.
Thus, policies are trained to lower the constraint excessively in the other tasks, especially in the Jackal task.
CPO shows high average CVs in all tasks, which can decrease score. 
For this reason, CPO shows the lowest score in the doggo task but shows the second highest performance in the other tasks, which can be attributed to trust-region methods that do not require Lagrange multipliers.
However, the number of failures in Table \ref{table:real environment} is four in the real-world environment, hence, it seems difficult to apply CPO to real robots.
In the case of TRPO-L, it can be inferred that Lagrange multipliers are updated unstably because CV curves in Figure \ref{fig:total results} fluctuate significantly compared to the other methods, and this fluctuation appears to have a negative effect on scores.
For PG-CVaR and CPPO, the CV curves also fluctuate largely, and the scores do not increase, as shown in Figure \ref{fig:total results}.
Both methods estimate CVaR from sampling and use the Lagrangian method to integrate the CVaR constraints into the policy objectives.
It seems that the variance of the CVaR estimation due to sampling cause unstable training.
Finally, in Figure \ref{sfig:ablation result}, MC shows the score curve similar to GAE, but it shows high average CVs due to the large variance of target values.
TD shows the lowest score, which can be attributed to the bias in the policy gradients.
Therefore, we can conclude that the GAE helps to train the policy and value functions by appropriately adjusting bias and variance.

\begin{table}[t]
\centering
\resizebox{\columnwidth}{!}{%
\begin{tabular}{|l|lll|lll|}
\hline
\multirow{2}{*}{} & \multicolumn{3}{l|}{Score $\uparrow$}                                                              & \multicolumn{3}{l|}{CV (CVaR) $\downarrow$}                                                                                     \\ \cline{2-7} 
                  & \multicolumn{1}{l|}{Point}         & \multicolumn{1}{l|}{Car}           & Doggo         & \multicolumn{1}{l|}{Point}                  & \multicolumn{1}{l|}{Car}                    & Doggo                  \\ \hline
TRC (proposed)    & \multicolumn{1}{l|}{\textbf{12.9}} & \multicolumn{1}{l|}{\textbf{15.6}} & \textbf{11.1} & \multicolumn{1}{l|}{\textbf{0.003 (0.017)}} & \multicolumn{1}{l|}{\textbf{0.002 (0.016)}} & 0.004 (0.019)          \\ \hline
CPO               & \multicolumn{1}{l|}{8.3}           & \multicolumn{1}{l|}{12.9}          & 6.0           & \multicolumn{1}{l|}{0.020 (0.058)}          & \multicolumn{1}{l|}{0.011 (0.041)}          & 0.021 (0.056)          \\ \hline
WCSAC             & \multicolumn{1}{l|}{-0.2}          & \multicolumn{1}{l|}{5.6}           & 10.5          & \multicolumn{1}{l|}{0.016 (0.167)}          & \multicolumn{1}{l|}{0.003 (0.067)}          & \textbf{0.001 (0.009)} \\ \hline
TRPO-L            & \multicolumn{1}{l|}{6.1}           & \multicolumn{1}{l|}{4.2}           & -             & \multicolumn{1}{l|}{0.022 (0.060)}          & \multicolumn{1}{l|}{0.023 (0.102)}          & -                      \\ \hline
PG-CVaR           & \multicolumn{1}{l|}{0.0}           & \multicolumn{1}{l|}{-0.4}          & -             & \multicolumn{1}{l|}{0.026 (0.158)}          & \multicolumn{1}{l|}{0.024 (0.135)}          & -                      \\ \hline
CPPO              & \multicolumn{1}{l|}{-1.7}          & \multicolumn{1}{l|}{-1.2}          & -             & \multicolumn{1}{l|}{0.019 (0.124)}          & \multicolumn{1}{l|}{0.006 (0.052)}          & -                      \\ \hline
\end{tabular}
}
\caption{Final policy performance of the simulation experiments.
The average score is presented in the left columns, and the mean and CVaR of the CV divided by episode length are presented in the right columns.}
\label{table:final policy}
\vspace{-5pt}
\end{table}

\begin{table}[t]
\centering
\resizebox{0.75\columnwidth}{!}{%
\begin{tabular}{|l|l|l|l|l|}
\hline
      & Reward sum $\uparrow$ & CV (CVaR) $\downarrow$ & \multicolumn{2}{l|}{\# of failures $\downarrow$} \\ \hline
TRC (proposed)   & \textbf{12.5}       & 0.004 (0.024)     & \multicolumn{2}{l|}{0 / 10}         \\ \hline
CPO   & 3.2        & 0.276 (0.704)  & \multicolumn{2}{l|}{4 / 10}        \\ \hline
WCSAC & -0.6       & \textbf{0.000 (0.000)}      & \multicolumn{2}{l|}{0 / 10}         \\ \hline
\end{tabular}
}
\caption{Evaluation results in the real-world Jackal environment. The results are obtained by performing 10 episodes for each method and the number of episodes that ended up hitting an obstacle is indicated as the number of failures.}
\label{table:real environment}
\vspace{-10pt}
\end{table}

\section{CONCLUSIONS}

In order to train a safe policy even in the worst case, we have proposed a trust region-based method, called \emph{TRC}, which maximizes the return while satisfying the safety constraints on CVaR.
For the development of TRC, we have derived the upper bound on CVaR and proposed a policy update rule based on trust-region methods.
Additionally, GAEs for CVaR-related value networks are formulated to train the networks with low variance.
With the proposed policy and value network update rule, TRC outperforms existing safe RL methods in simulation and sim-to-real experiments while successfully satisfying all constraints during training. 
We plan to apply TRC to safe manipulation and locomotion tasks in our future work.

\appendix
\small
This section provides proofs of Theorem \ref{theorem:cost square sum}, Theorem \ref{theorem:upper bound} and derivation of GAE for a square function.
The necessary lemmas and corollaries are first presented.

\subsection{Preliminary}

\begin{lemma}
\label{lemma:doubly discounted state dist}
Assume that the state space is finite and the following equation holds for any function {\footnotesize$f$}: 
\begin{equation}
\scriptsize
\begin{aligned}
(1 - \gamma^2)\underset{s \sim \rho}{\mathbb{E}}\left[f(s)\right] + \gamma^2\underset{d_2^{\pi}, \pi, \mathcal{P}}{\mathbb{E}}\left[f(s')\right] - \underset{s \sim d_2^{\pi}}{\mathbb{E}}\left[f(s)\right] = 0.
\end{aligned}
\end{equation}
\end{lemma}

\begin{proof}
Let {\footnotesize$ P_{\pi}\in\mathbb{R}^{|\mathcal{S}|\times|\mathcal{S}|}$} is a state to next state transition probability matrix whose element is {\footnotesize$P_{\pi,ij}=\sum_{a\in\mathcal{A}}\pi(a|s_i)\mathcal{P}(s_j|s_i,a)$}.
\begin{equation}
\label{eq:d2 matrix}
\scriptsize
\begin{aligned}
&d_2^{\pi} = (1 - \gamma^2)(I - \gamma^2P_{\pi})^{-1}\rho \\
\Leftrightarrow &(I - \gamma^2P_{\pi})d_2^{\pi} = (1 - \gamma^2)\rho \\
\Leftrightarrow &(I - \gamma^2P_{\pi})\langle d_2^{\pi}, f\rangle = (1 - \gamma^2)\langle\rho, f\rangle \\
\Leftrightarrow &(1 - \gamma^2)\underset{s \sim \rho}{\mathbb{E}}\left[f(s)\right] + \gamma^2\underset{d_2^{\pi},\pi,\mathcal{P}}{\mathbb{E}}\left[f(s')\right] - \underset{s \sim d_2^{\pi}}{\mathbb{E}}\left[f(s)\right] = 0.\\
\end{aligned}
\vspace{-10 pt}
\end{equation}
\end{proof}

\noindent
Using (\ref{eq:expectation of the discounted cost sum}) and Lemma \ref{lemma:doubly discounted state dist}, a corollary is derived as follows:

\begin{corollary}
For any stochastic policy {\footnotesize$\pi$} and function {\footnotesize$f$}, the following equation holds:
\label{corollary:square cost function}
\begin{equation}
\scriptsize
\begin{aligned}
J_{S}(\pi) = \underset{s \sim \rho}{\mathbb{E}}&\left[f(s)\right] + \frac{1}{1-\gamma^2}\underset{d_2^{\pi},\pi,\mathcal{P}}{\mathbb{E}}\left[C(s, a, s')^2 + \right. \\
&\left. 2\gamma C(s, a, s')V_C^{\pi}(s') + \gamma^2f(s') - f(s)\right]. \\
\end{aligned}
\end{equation}
\end{corollary}

\begin{lemma}
\label{lemma:inf norm of cost value}
For any policy {\footnotesize$\pi'$} and {\footnotesize$\pi$},
\begin{equation}
\scriptsize
\begin{aligned}
\left\lVert V_C^{\pi'} - V_C^{\pi}\right\rVert_{\infty} \leq \frac{2\left\lVert V_C^{\pi}\right\rVert_{\infty}}{1-\gamma}D^{\pi}(\pi').
\end{aligned}
\end{equation}
\end{lemma}

\begin{proof}
Define an expected cost vector {\footnotesize$C_s^a \in \mathbb{R}^{|\mathcal{A}|}$}, where {\footnotesize$C_{s,i}^a = \underset{s' \sim \mathcal{P}}{\mathbb{E}}\left[c(s, a_i, s')\right]$}, and a transition matrix {\footnotesize$P_s \in \mathbb{R}^{|\mathcal{A}| \times |\mathcal{S}|}$}, where {\footnotesize$P_{s, ij} = \mathcal{P}(s_j|s, a_i)$}.  
Then, {\footnotesize$V_C^{\pi}(s) = \langle\pi(s), C_s^a + \gamma P_s V_C^{\pi}\rangle$}.
{\scriptsize
\begin{align}
\Rightarrow& V_C^{\pi'}(s) - V_C^{\pi}(s) \nonumber\\
&= \langle\pi'(s), C_s^a + \gamma P_s V_C^{\pi'}\rangle - \langle\pi(s), C_s^a + \gamma P_s V_C^{\pi}\rangle \nonumber\\
&= \langle\pi'(s), C_s^a + \gamma P_s (V_C^{\pi'} - V_C^{\pi})\rangle + \langle\pi'(s) - \pi(s), C_s^a + \gamma P_s V_C^{\pi}\rangle. \nonumber\\
\Rightarrow& \left\lVert V_C^{\pi'} - V_C^{\pi}\right\rVert_{\infty} \\
&= \underset{s}{\mathrm{max}}\left[\langle\pi'(s), \gamma P_s (V_C^{\pi'} - V_C^{\pi})\rangle + \langle\pi'(s) - \pi(s), C_s^a + \gamma P_s V_C^{\pi}\rangle\right] \nonumber\\
&\leq \underset{s}{\mathrm{max}}\langle\pi'(s), \gamma P_s (V_C^{\pi'} - V_C^{\pi})\rangle + \underset{s}{\mathrm{max}}\langle\pi'(s) - \pi(s), C_s^a + \gamma P_s V_C^{\pi}\rangle. \nonumber
\end{align}
}
According to H{\"o}lder’s inequality, {\footnotesize$\langle a, b\rangle \leq \left\lVert a\right\rVert_{1}\left\lVert b\right\rVert_{\infty}$}.
\begin{equation}
\scriptsize
\begin{aligned}
\Rightarrow& \left\lVert V_C^{\pi'} - V_C^{\pi}\right\rVert_{\infty} \leq \underset{s}{\mathrm{max}}\left\lVert\pi'(s)\right\rVert_{1}\left\lVert \gamma P_s (V_C^{\pi'} - V_C^{\pi})\right\rVert_{\infty} + \\
&\qquad\qquad\qquad\qquad \underset{s}{\mathrm{max}}\left\lVert\pi'(s) - \pi(s)\right\rVert_{1}\left\lVert C_s^a + \gamma P_s V_C^{\pi}\right\rVert_{\infty}\\
&\leq \gamma\left\lVert V_C^{\pi'} - V_C^{\pi}\right\rVert_{\infty} + 2\left\lVert V_C^{\pi}\right\rVert_{\infty}\underset{s}{\mathrm{max}}D_{TV}(\pi'||\pi)[s] \\
\Rightarrow& \left\lVert V_C^{\pi'} - V_C^{\pi}\right\rVert_{\infty} \leq \frac{2\left\lVert V_C^{\pi}\right\rVert_{\infty}}{1-\gamma}D^{\pi}(\pi').
\end{aligned}
\vspace{-10pt}
\end{equation}
\end{proof}

\begin{lemma}
\label{lemma:total variance}
For any policy {\footnotesize$\pi'$} and {\footnotesize$\pi$}, the following inequality holds:
\begin{equation}
\scriptsize
\begin{aligned}
D_{TV}(d_2^{\pi'}||d_2^{\pi}) \leq \frac{\gamma^2}{1 - \gamma^2}D^{\pi}(\pi').
\end{aligned}
\end{equation}
\end{lemma}

\begin{proof}
Let {\footnotesize$G_{\pi}:=\left(I-\gamma^2 P_{\pi}\right)^{-1}$}.
Then, the following is derived:
\begin{equation}
\scriptsize
\begin{aligned}
G_{\pi}^{-1} - G_{\pi'}^{-1} &= \gamma^2 (P_{\pi'} - P_{\pi}) \\
\Leftrightarrow \; G_{\pi'} - G_{\pi} &= \gamma^2 G_{\pi'}(P_{\pi'} - P_{\pi})G_{\pi}. \\
\end{aligned}
\end{equation}
With the above result and (\ref{eq:d2 matrix}), 
\begin{equation}
\label{eq:lemma3_1}
\scriptsize
\begin{aligned}
\lVert d_2^{\pi'} - d_2^{\pi} \rVert_1 &= (1 - \gamma^2)\lVert(G_{\pi'} - G_{\pi})\rho\rVert_1 \\
&= \gamma^2(1 - \gamma^2)\lVert G_{\pi'}(P_{\pi'} - P_{\pi})G_{\pi}\rho\rVert_1 \\
&= \gamma^2\lVert G_{\pi'}(P_{\pi'} - P_{\pi})d_2^{\pi}\rVert_1 \\
&\leq \gamma^2\lVert G_{\pi'}\rVert_1\cdot\lVert(P_{\pi'} - P_{\pi})d_2^{\pi}\rVert_1. \\
\end{aligned}
\end{equation}

\begin{equation}
\label{eq:lemma3_2}
\scriptsize
\begin{aligned}
&\lVert (P_{\pi'} - P_{\pi})d_2^{\pi} \rVert_1 = \\
&\quad \sum_{s'\in\mathcal{S}}\left|\sum_{s\in\mathcal{S}}d_2^{\pi}(s)\cdot\left(\sum_{a \in \mathcal{A}}(\pi(a|s) - \pi'(a|s))\mathcal{P}(s'|s,a)\right) \right| \\
&\leq \sum_{s'\in\mathcal{S}}\sum_{s\in\mathcal{S}}d_2^{\pi}(s)\sum_{a \in \mathcal{A}}\left|\pi(a|s) - \pi'(a|s)\right|\mathcal{P}(s'|s,a) \\
&= 2\underset{s \sim d_2^{\pi}}{\mathbb{E}}\left[D_{\mathrm{TV}}(\pi'||\pi)[s]\right] \leq 2\underset{s}{\mathrm{max}}D_{TV}(\pi'||\pi)[s].\\
\end{aligned}
\end{equation}
With (\ref{eq:lemma3_1}) and (\ref{eq:lemma3_2}),

\begin{equation}
\scriptsize
\begin{aligned}
D_{TV}(d_2^{\pi'}||d_2^{\pi}) &= \frac{\lVert d_2^{\pi'}-d_2^{\pi}\rVert_1}{2} \leq \frac{\gamma^2\lVert G_{\pi'}\rVert_1\cdot\lVert(P_{\pi'} - P_{\pi})d_2^{\pi}\rVert_1}{2} \\
&\leq \frac{\gamma^2}{1 - \gamma^2}D^{\pi}(\pi').
\end{aligned}
\vspace{-10 pt}
\end{equation}
\end{proof}

\subsection{Proof of Theorem \ref{theorem:cost square sum}}

\begin{proof} 
Define {\footnotesize${\delta_f^{\pi'}(s) := \underset{\begin{subarray}{c}a \sim \pi'\\s' \sim \mathcal{P}\end{subarray}}{\mathbb{E}}\left[C(s, a, s')^2 + \gamma^2f(s') - f(s) \; + \right.}$ ${\left. 2\gamma c(s,a,s')V_C^{\pi'}(s')\right]}$}. 
With {\footnotesize$\delta_f^{\pi'}$} and Corollary \ref{corollary:square cost function},
\begin{equation}
\label{eq:theorem1_1}
\scriptsize
\begin{aligned}
&\quad J_{S}(\pi') = \underset{s \sim \rho}{\mathbb{E}}\left[f(s)\right] + \frac{1}{1 - \gamma^2}\langle d_2^{\pi'}, \delta_f^{\pi'}\rangle. \\
&\Rightarrow J_{S}(\pi') - J_{S}(\pi) = \frac{1}{1 - \gamma^2}\left(\langle d_2^{\pi'}, \delta_f^{\pi'}\rangle - \langle d_2^{\pi}, \delta_f^{\pi}\rangle \right) \\
& \leq \frac{1}{1 - \gamma^2}\langle d_2^{\pi}, \delta_f^{\pi'} - \delta_f^{\pi}\rangle + \frac{2\lVert\delta_f^{\pi'}\rVert_{\infty}}{1 - \gamma^2} D_{TV}(d_2^{\pi'}||d_2^{\pi}). \\
\end{aligned}
\end{equation}
Substituting {\footnotesize$S_C^{\pi}$} for {\footnotesize$f$} and using Lemma \ref{lemma:inf norm of cost value},
\begin{equation}
\label{eq:theorem1_2}
\scriptsize
\begin{aligned}
&\langle d_2^{\pi}, \delta_f^{\pi'}\rangle = \\
&\underset{d_2^{\pi},\pi',\mathcal{P}}{\mathbb{E}}\left[C(s, a, s')^2 + \gamma^2S_C^{\pi}(s') + 2\gamma C(s, a, s')V_C^{\pi'}(s') - S_C^{\pi}(s) \right] \\
&\leq \underset{d_2^{\pi},\pi',\mathcal{P}}{\mathbb{E}}\left[C(s, a, s')^2 + \gamma^2S_C^{\pi}(s') + 2\gamma C(s, a, s')V_C^{\pi}(s') \;- \right.\\
&\quad S_C^{\pi}(s) + \left. \frac{4\gamma C(s,a,s')\left\lVert V_C^{\pi}\right\rVert_{\infty}}{1 - \gamma}D^{\pi}(\pi')\right] \\
&= \underset{\begin{subarray}{c}s \sim d_2^{\pi}\\a \sim \pi'\end{subarray}}{\mathbb{E}}\left[A_S^{\pi}(s,a)\right] + \frac{4\gamma\left\lVert V_C^{\pi}\right\rVert_{\infty}}{1-\gamma}\underset{d_2^{\pi},\pi',\mathcal{P}}{\mathbb{E}}\left[C(s,a,s')\right]D^{\pi}(\pi'). \\
&= \underset{\begin{subarray}{c}s \sim d_2^{\pi}\\a \sim \pi\end{subarray}}{\mathbb{E}}\left[\frac{\pi'(a|s)}{\pi(a|s)} A_S^{\pi}(s,a)\right] + \frac{4\gamma\left\lVert V_C^{\pi}\right\rVert_{\infty}}{1-\gamma}\underset{d_2^{\pi},\pi',\mathcal{P}}{\mathbb{E}}\left[C(s,a,s')\right] D^{\pi}(\pi'). \\
\end{aligned}
\end{equation}
With (\ref{eq:theorem1_1}), (\ref{eq:theorem1_2}), and Lemma \ref{lemma:total variance},
{\scriptsize
\begin{align*}
&J_{S}(\pi') - J_{S}(\pi) \\
&\leq \frac{1}{1 - \gamma^2}\left(\langle d_2^{\pi}, \delta_f^{\pi'}\rangle + \frac{2\gamma^2\lVert\delta_f^{\pi'}\rVert_{\infty}}{1 - \gamma^2} D^{\pi}(\pi')\right) \\
&\leq \frac{1}{1-\gamma^2}\underset{\begin{subarray}{c}s \sim d_2^{\pi}\\a \sim \pi\end{subarray}}{\mathbb{E}}\left[\frac{\pi'(a|s)}{\pi(a|s)} A_S^{\pi}(s,a)\right] + \\
&\quad \frac{2}{1-\gamma^2}D^{\pi}(\pi')\left(\frac{\gamma^2}{1 - \gamma^2}\underset{s}{\mathrm{max}}\underset{a \sim \pi'}{\mathbb{E}}\left[A_{S}^{\pi}(s,a)\right] + \right. \\
&\quad \left. \frac{2\gamma\underset{s}{\mathrm{max}}\left|V_C^{\pi}(s)\right|}{1-\gamma}\underset{d_2^{\pi},\pi',\mathcal{P}}{\mathbb{E}}\left[C(s,a,s')\right]\right) \tag{\stepcounter{equation}\theequation}\\
&= \frac{1}{1-\gamma^2}\underset{\begin{subarray}{c}s \sim d_2^{\pi}\\a \sim \pi\end{subarray}}{\mathbb{E}}\left[\frac{\pi'(a|s)}{\pi(a|s)}A_{S}^{\pi}(s,a) \right] + \frac{2\epsilon_S^{\pi'}}{1-\gamma^2}D^{\pi}(\pi').
\vspace{-10pt}
\end{align*}}
\end{proof}

\subsection{Proof of Theorem \ref{theorem:upper bound}}

\begin{proof}
The following inequality holds by Theorem 1 in \cite{achiam2017constrained}:
\begin{equation}
\label{eq:cpo inequality}
\scriptsize
\begin{aligned}
J_C^{\pi}(\pi') - \frac{2\gamma\epsilon_C^{\pi'}}{(1-\gamma)^2}D^{\pi}(\pi') \leq J_C(\pi') \leq J_C^{\pi}(\pi') + \frac{2\gamma\epsilon_C^{\pi'}}{(1-\gamma)^2}D^{\pi}(\pi').
\end{aligned}
\end{equation}
The following inequality also holds in that the cost function is defined on a set of nonnegative real numbers:
\begin{equation}
\label{eq:square inequality for thm4.2}
\scriptsize
\begin{aligned}
J_C^{\pi}(\pi')^2 - J_C(\pi')^2 &\leq
\frac{4\gamma\epsilon_C^{\pi'}J_C^{\pi}(\pi')}{(1-\gamma)^2}D^{\pi}(\pi') - \left(\frac{2\gamma\epsilon_C^{\pi'}}{(1-\gamma)^2}D^{\pi}(\pi')\right)^2.
\end{aligned}
\end{equation}
Then, the following inequality is derived with (\ref{eq:cpo inequality}), (\ref{eq:square inequality for thm4.2}), and Theorem \ref{theorem:cost square sum}:
\begin{equation}
\label{eq: squart inequality for thm4.2}
\scriptsize
\begin{aligned}
&\sqrt{J_S(\pi') - J_C(\pi')^2} - \sqrt{J_S^{\pi}(\pi') - J_C^{\pi}(\pi')^2} \\
&= \frac{J_S(\pi') - J_S^{\pi}(\pi') + J_C^{\pi}(\pi')^2 - J_C(\pi')^2}{\sqrt{J_S(\pi') - J_C(\pi')^2} + \sqrt{J_S^{\pi}(\pi') - J_C^{\pi}(\pi')^2}}\\
&\leq \frac{J_S(\pi') - J_S^{\pi}(\pi') + J_C^{\pi}(\pi')^2 - J_C(\pi')^2}{\sqrt{J_S(\pi') - J_C(\pi')^2}}\\
&\leq \frac{2\epsilon_{\mathrm{CVaR}}^{\pi'}/(1-\gamma^2)}{{\sqrt{J_S(\pi') - J_C(\pi')^2}}} D^{\pi}(\pi'). \\
\end{aligned}
\end{equation}
The upper bound is then derived by taking the weighted sum of the two inequalities (\ref{eq:cpo inequality}), (\ref{eq: squart inequality for thm4.2}):
\begin{equation}
\scriptsize
\begin{aligned}
&\mathrm{CVaR}_{\alpha}(C_{\pi'}) \\
&= J_{C}(\pi') + \frac{\phi(\Phi^{-1}(\alpha))}{\alpha}\sqrt{J_{S}(\pi') - J_{C}(\pi')^2} \\
&\leq J_C^{\pi}(\pi') + \frac{\phi(\Phi^{-1}(\alpha))}{\alpha}\sqrt{J_S^{\pi}(\pi') - J_C^{\pi}(\pi')^2} \;+ \\
&\qquad \frac{2}{1-\gamma}\left(\frac{\gamma\epsilon_C^{\pi'}}{1-\gamma} + \frac{\phi(\Phi^{-1}(\alpha))/\alpha}{{\sqrt{J_S(\pi') - J_C(\pi')^2}}}\frac{\epsilon_{\mathrm{CVaR}}^{\pi'}}{1+\gamma}\right) D^{\pi}(\pi').\\
\end{aligned}
\end{equation}
\end{proof}
\subsection{GAE for Square Function}

First, $k$-step square advantage is defined as follows:
\begin{equation}
\scriptsize
\label{eq: k-step cost square advantage}
\begin{aligned}
A_{S,t}^{(k)} = \underset{\begin{subarray}{c}\pi, \mathcal{P}\end{subarray}}{\mathbb{E}}&\left[\left(\sum_{i=t}^{t+k-1}\gamma^{i-t}C_i\right)^2 + 2\gamma^k\left(\sum_{i=t}^{t+k-1}\gamma^{i-t}C_i\right)V_{C,t+k}^{\pi} \right.\\
&\left. +\; \gamma^{2k}S_{C,t+k}^{\pi}\big|s_t\right] - S_{C,t}^{\pi}, \\
\end{aligned}
\end{equation}
where {\footnotesize$V_{C,t}^{\pi}=V_C^{\pi}(s_t)$} and {\footnotesize$S_{C,t}^{\pi}=S_C^{\pi}(s_t)$}.
The difference in the advantage of adjacent time steps is derived as follows:
\begin{equation}
\scriptsize
\begin{aligned}
&A_{S,t}^{(k+1)} - A_{S,t}^{(k)} \\
&=\gamma^{2k}\underset{\pi, \mathcal{P}}{\mathbb{E}}\left[C_{t+k}^2 + 2\gamma C_{t+k}V_{C,t+k+1}^{\pi} + \gamma^2 S_{C,t+k+1}^{\pi} - S_{C,t+k}^{\pi}\right]. \\
\end{aligned}
\end{equation}
Then, the $k$-step advantage can be expressed as {\small$A_{S,t}^{(k)}=\underset{\pi, \mathcal{P}}{\mathbb{E}}\left[\sum_{i=t}^{t+k-1}\gamma^{2(i-t)}\delta_i^S\right]$}.
Finally, as in \cite{schulman2015high}, the GAE for the cost square function is defined as:
{\footnotesize $\hat{A}_{S,t}^{\mathrm{GAE}(\gamma, \lambda)} := \sum_{i=t}^{\infty}(\gamma^2\lambda)^{i-t}\delta_i^S$}.

\addtolength{\textheight}{-12cm}  
\balance

\bibliographystyle{IEEEtran}
\bibliography{main}

\end{document}